\def\year{2023}\relax
\documentclass[letterpaper]{article}
\usepackage{aaai22}
\usepackage{times}
\usepackage{helvet}
\usepackage{courier}
\usepackage[hyphens]{url}
\usepackage[colorlinks, linkcolor=red, anchorcolor=red, citecolor=cyan, breaklinks]{hyperref}
\usepackage{graphicx}

\usepackage{natbib}
\usepackage{caption}
\usepackage{bbm}
\usepackage{float}
\usepackage{stfloats}
\usepackage{graphicx}
\usepackage{amsmath}
\usepackage{multirow}
\usepackage{amssymb}
\usepackage{booktabs}
\usepackage{subfiles}
\usepackage{subfigure}
\usepackage[switch]{lineno}

\usepackage{algorithm}
\usepackage{algpseudocode}

\newtheorem{proof}{Proof}
\newtheorem{theorem}{Proposition}

\title{Extracting Low-/High- Frequency Knowledge from Graph Neural Networks and Injecting it into MLPs: An Effective GNN-to-MLP Distillation Framework}

\author{
Lirong Wu {$^{1,2,^*}$}, Haitao Lin{$^{1,2,^*}$}, Yufei Huang{$^{1,2}$}, Tianyu Fan{$^{2}$}, Stan.Z.Li{$^{1,^\dagger}$}\\
}
\affiliations{
$^1$ AI Division, School of Engineering, Westlake University, Hangzhou, 310030 \\ 
$^2$ Zhejiang University, Hangzhou, 310058 \\
\{wulirong, linhaitao, huangyufei, fantianyu, stan.zq.li\}@westlake.edu.cn}

\begin{document}
\maketitle

\begin{abstract}
Recent years have witnessed the great success of Graph Neural Networks (GNNs) in handling graph-related tasks. However, MLPs remain the primary workhorse for practical industrial applications due to their desirable inference efficiency and scalability. To reduce their gaps, one can directly distill knowledge from a well-designed teacher GNN to a student MLP, which is termed as GNN-to-MLP distillation. However, the process of distillation usually entails a loss of information, and \textit{``which knowledge patterns of GNNs are more likely to be left and distilled into MLPs?"} becomes an important question. In this paper, we first factorize the knowledge learned by GNNs into low- and high-frequency components in the spectral domain and then derive their correspondence in the spatial domain. Furthermore, we identified a potential \emph{information drowning} problem for existing GNN-to-MLP distillation, i.e., the high-frequency knowledge of the pre-trained GNNs may be overwhelmed by the low-frequency knowledge during distillation; we have described in detail what it represents, how it arises, what impact it has, and how to deal with it. In this paper, we propose an efficient \textit{Full-Frequency GNN-to-MLP} (FF-G2M) distillation framework, which extracts both low-frequency and high-frequency knowledge from GNNs and injects it into MLPs. Extensive experiments show that FF-G2M improves over the vanilla MLPs by 12.6\% and outperforms its corresponding teacher GNNs by 2.6\% averaged over six graph datasets and three common GNN architectures. Codes are publicly available at: \url{https://github.com/LirongWu/FF-G2M}.
\end{abstract}

\vspace{-1em}
\section{Introduction} \label{sec:1}
In many real-world applications, including social networks, chemical molecules, and citation networks, data can be naturally modeled as graphs. Recently, the emerging Graph Neural Networks (GNNs) \cite{hamilton2017inductive,kipf2016semi,velivckovic2017graph,wu2021self,wu2022multi,liu2020towards,zhou2020graph,liu2022gradients,xia2022simgrace} have demonstrated their powerful capability to handle various graph-related tasks \cite{zhang2018link,fan2019graph,errica2019fair,wu2021graphmixup}. However, practical deployments of GNNs in the industry are still less popular due to inference efficiency and scalability challenges incurred by data dependency \cite{jia2020redundancy,zhang2021graph}. In other words, GNNs generally rely on message passing to aggregate features from the neighborhood, but fetching and aggregating these nodes during inference can burden latency-sensitive applications. In contrast, Multi-Layer Perceptrons (MLPs) involve no data dependence between pairs of nodes and infer much faster than GNNs, but often with less competitive performance. Motivated by these complementary strengths and weaknesses, one solution to reduce their gaps is to perform GNN-to-MLP knowledge distillation \cite{yang2021extract,zhang2021graph,ghorbani2021gkd,gou2021knowledge}, which extracts knowledge from a well-trained teacher GNN and then distills it into a student MLP with the same network architecture (e.g., layer number and layer size).

Most of the existing GNN-to-MLP distillation methods \cite{yang2021extract,zhang2021graph,ghorbani2021gkd} focus on special designs on either student MLPs or teacher GNNs, but default to distill knowledge in a node-to-node fashion. For example, CPF \cite{yang2021extract} combines Label Propagation (LP) \cite{iscen2019label} into the student MLPs to improve classification performance and thus still suffers from the neighborhood-fetching latency caused by label propagation, which defeats the original intention of MLPs to be inference-efficient. In contrast, GLNN \cite{zhang2021graph} directly distills knowledge from arbitrary GNNs to vanilla MLPs with the same network architecture. While the distilled MLPs of GLNN can be greatly improved by employing more powerful teacher GNNs, the process of distillation usually entails a loss of information \cite{kim2021comparing}, which may lead to sub-optimal student MLPs. In this paper, we look away from specific instantiations of teacher GNNs and student MLPs, but rather explore two fundamental questions: (1) \textit{Can existing GNN-to-MLP distillation ensure that sufficient knowledge is distilled from teacher GNNs to student MLPs?} If not, (2) \textit{Which knowledge patterns of GNNs are more likely to be distilled into student MLPs?} 

\textbf{Present Work.} In this paper, we identify a potential \emph{information drowning} problem for existing GNN-to-MLP distillation, i.e., the high-frequency knowledge of the pre-trained GNNs may be overwhelmed by the low-frequency knowledge during distillation. To illustrate this, we first factorize GNN knowledge into low- and high-frequency components using graph signal processing theory in the spectral domain and then derive their correspondence in the spatial domain. Moreover, we conduct a comprehensive investigation of the roles played by low- and high-frequency components in the distillation process and describe in detail what \emph{information drowning} represents, how it arises, what impact it has, and how to deal with it. Extensive experiments have shown that high-frequency and low-frequency knowledge are complementary to each other, and they can further improve performance on top of each other. In this paper, we propose a novel \textit{Full-Frequency GNN-to-MLP} (FF-G2M) distillation framework, which extracts both low- and high-frequency knowledge from teacher GNNs and injects it into student MLPs.

\vspace{-0.5em}
\section{Related Work}
\textbf{Graph Neural Networks (GNNs).}
The early GNNs define graph convolution kernels in the spectral domain \cite{bruna2013spectral,defferrard2016convolutional} based on the graph signal processing theory, known as ChebyNet \cite{defferrard2016convolutional} and Graph Convolutional Networks (GCN) \cite{kipf2016semi}. The later GNNs directly define updating rules in the spatial space and focus on the design of neighborhood aggregation functions. For instance, GraphSAGE \cite{hamilton2017inductive} employs a generalized induction framework to generate embeddings for previously unseen nodes by aggregating known node features. Moreover, GAT \cite{velivckovic2017graph} introduces the self-attention mechanism to assign different importance scores to neighbors for better information aggregation. We refer interested readers to the surveys \cite{liu2020towards,wu2020comprehensive,zhou2020graph} for more GNN architectures. 
\newline
\vspace{-0.8em}

\noindent \textbf{Graph Knowledge Distillation.}
Despite the great progress, most existing GNNs share the de facto design that relies on message passing to aggregate features from neighborhoods, which may be one major source of latency in GNN inference. To address this problem, there are previous works that attempt to distill knowledge from large teacher GNNs to smaller student GNNs, termed as GNN-to-GNN \cite{lassance2020deep,wu2022knowledge,ren2021multi,joshi2021representation}. For example, the student model in RDD \cite{zhang2020reliable} and TinyGNN \cite{yan2020tinygnn} is a GNN with fewer parameters but not necessarily fewer layers than the teacher GNN, which makes both designs still suffer from the neighborhood-fetching latency caused by data dependency. 
\newline
\vspace{-0.8em}

To enjoy the low-latency of MLPs and high-accuracy of GNNs, the other branch of graph knowledge distillation is to directly distill from large teacher GNNs to student MLPs, termed as GNN-to-MLP. The existing work on GNN-to-MLP distillation can be mainly divided into two branches: student MLPs-focused and teacher GNNs-focused. The former branch, such as CPF \cite{yang2021extract}, \textit{directly} improves student MLPs by adopting deeper and wider network architectures or incorporating label propagation, both of which burden the inference latency. The other branch, such as GLNN \cite{zhang2021graph}, distills knowledge from teacher GNNs to vanilla MLPs with the same network architectures but without other computing-consuming operations; while the performance of their distilled MLPs can be \textit{indirectly} improved by employing more powerful GNNs, they still cannot match their corresponding teacher GNNs. Moreover, PGKD \cite{wu2023edge} proposes a Prototype-Guided Knowledge Distillation~(PGKD) method, which does not require graph edges yet learns structure-aware MLPs. In this paper, we aim to develop a \textbf{model-agnostic} GNN-to-MLP distillation that is applicable to various GNN architectures.

\vspace{-0.8em}
\section{Preliminaries}
\textbf{Notions.} 
Let $\mathcal{G}=(\mathcal{V}, \mathcal{E}, \mathbf{X})$ be an attributed graph, where $\mathcal{V}$ is the set of $N$ nodes with features $\mathbf{X}=\left[\mathbf{x}_{1}, \mathbf{x}_{2}, \cdots, \mathbf{x}_{N}\right]\in \mathbb{R}^{N \times d}$ and $\mathcal{E}$ denotes the edge set. Each node $v_i \in \mathcal{V}$ is associated with a $d$-dimensional features vector $\mathbf{x}_{i}$, and each edge $e_{i, j} \in \mathcal{E}$ denotes a connection between node $v_i$ and $v_j$. The graph structure is denoted by an adjacency matrix $\mathbf{A} \in[0,1]^{N \times N}$ with $\mathbf{A}_{i,j}=1$ if $e_{i,j}\in\mathcal{E}$ and $\mathbf{A}_{i,j}=0$ if $e_{i,j} \notin \mathcal{E}$. Consider a semi-supervised node classification task where only a subset of node $\mathcal{V}_L$ with labels $\mathcal{Y}_L$ are known, we denote the labeled set as $\mathcal{D}_L=(\mathcal{V}_L,\mathcal{Y}_L)$ and unlabeled set as $\mathcal{D}_U=(\mathcal{V}_U,\mathcal{Y}_U)$, where $\mathcal{V}_U=\mathcal{V} \backslash \mathcal{V}_L$. The node classification aims to learn a mapping $\Phi: \mathcal{V} \rightarrow \mathcal{Y}$ with labeled data, so that it can be used to infer the labels $\mathcal{Y}_U$.
\newline
\vspace{-0.8em}

\noindent\textbf{Graph Neural Networks (GNNs).}
Most existing GNNs rely on message passing to aggregate features from the neighborhood. A general GNN framework consists of two key computations for each node $v_i$: (1) $\operatorname{AGGREGATE}$: aggregating messages from neighborhood $\mathcal{N}_i$; (2) $\operatorname{UPDATE}$: updating node representation from its representation in the previous layer and aggregated messages. Considering a $L$-layer GNN, the formulation of the $l$-th layer is as follows
\vspace{-0.3em}
\begin{equation}
\begin{small}
\begin{aligned}
\mathbf{m}_{i}^{(l)}  =& \operatorname{AGGREGATE}^{(l)}\left(\big\{\mathbf{h}_{j}^{(l-1)}: v_{j} \in \mathcal{N}_i\big\}\right) \\ 
\mathbf{h}_{i}^{(l)}  =& \operatorname{UPDATE}^{(l)}\left(\mathbf{h}_{i}^{(l-1)}, \mathbf{m}_{i}^{(l)}\right)
\label{equ:1}
\end{aligned}
\end{small}
\end{equation}
where $1\leq l \leq L$, $\mathbf{h}_{i}^{(0)}=\mathbf{x}_{i}$ is the input feature, and $\mathbf{h}_{i}^{(l)}$ is the representation of node $v_i$ in the $l$-th layer.
\newline
\vspace{-0.8em}

\noindent\textbf{Multi-Layer Perceptrons (MLPs).}
To achieve efficient inference, the vanilla MLPs (with the same network architecture as the teacher GNNs) are used as the student model by default in this paper. For a $L$-layer MLP, the $l$-th layer is composed of a linear transformation, an activation function $\sigma=\mathrm{ReLu}(\cdot)$, and a dropout function $\operatorname{Dropout}(\cdot)$, as
\begin{equation}
\mathbf{z}^{(l)}_i=\operatorname{Dropout}\big(\sigma\big(\mathbf{z}^{(l-1)}_i \mathbf{W}^{(l-1)}\big)\big), \quad \mathbf{z}^{(0)}_i=\mathbf{x}_i
\label{equ:2}
\end{equation}
where $\mathbf{W}^{(0)} \in \mathbb{R}^{d \times F}$ and $\mathbf{W}^{(l)} \in \mathbb{R}^{F \times F}$ $(1 \leq l < L)$ are weight matrices with the hidden dimension $F$. In this paper, the network architecture of MLPs, such as the layer number $L$ and layer size $F$, is set the same as that of teacher GNNs.
\newline
\vspace{-0.8em}

\noindent\textbf{GNN-to-MLP Knowledge Distillation.}
The knowledge distillation is first introduced in \cite{hinton2015distilling} to handle mainly image data, where knowledge is transferred from a cumbersome teacher model to a simpler student model. The later works on GNN-to-MLP distillation \cite{yang2021extract,zhang2021graph,ghorbani2021gkd} extend it to the graph domain by imposing KL-divergence constraint $\mathcal{D}_{KL}(\cdot, \cdot)$ between the softmax label distributions generated by teacher GNNs and student MLPs and directly optimizing the objective function as follows
\begin{equation}
\begin{small}
\begin{aligned}
\mathcal{L}_{\mathrm{KD}}=\frac{1}{|\mathcal{V}|}\sum_{i\in\mathcal{V}}\mathcal{D}_{KL}\left(\operatorname{softmax}\big(\mathbf{z}_{i}^{(L)}\big), \operatorname{softmax}\big(\mathbf{h}_{i}^{(L)}\big)\right)
\label{equ:3}
\end{aligned}
\end{small}
\end{equation}

\section{Knowledge Factorization from the Perspective of Spectral and Spatial Domain}
In this section, we first theoretically factorize the knowledge learned by GNNs into low- and high-frequency components in the spectral domain based on graph signal processing theory \cite{shuman2013emerging}.
The normalized graph Laplacian matrix of graph $\mathcal{G}$ is defined as $\mathbf{L}\!=\!\mathbf{I}_{N}\!-\!\widetilde{\mathbf{D}}^{-\frac{1}{2}} \widetilde{\mathbf{A}} \widetilde{\mathbf{D}}^{-\frac{1}{2}}$, where $\widetilde{\mathbf{A}} \!=\! \mathbf{A} \!+\! \mathbf{I}_{N} \in \mathbb{R}^{N \times N}$ is an adjacency matrix with self-loop, $\widetilde{\mathbf{D}} \in \mathbb{R}^{N \times N}$ is a diagonal degree matrix with $\widetilde{\mathbf{D}}_{i, i}=\sum_{j} \widetilde{\mathbf{A}}_{i, j}$, and $\mathbf{I}_{N}$ denotes the identity matrix. Since $\mathbf{L}$ is a real symmetric matrix, it can be eigendecomposed as $\mathbf{L}=\mathbf{U} \Lambda \mathbf{U}^{\top}$, where $\Lambda=\operatorname{diag}\left(\left[\lambda_{1}, \lambda_{2}, \cdots, \lambda_{N}\right]\right)$ with each eigenvalue $\lambda_{l} \in[0,2]$ corresponding to an eigenvectors $\mathbf{u}_{l}$ in $\mathbf{U}$ \cite{chung1997spectral}. According to graph signal processing theory, we can directly take the eigenvector $\left\{\mathbf{u}_{l}\right\}_{l=1}^{N}$ as bases. Given signal $\mathbf{x} \in \mathbb{R}^{d}$, the graph Fourier transform and inverse Fourier transform \cite{sandryhaila2013discrete,ricaud2019fourier} are defined as $\widehat{\mathbf{x}}=\mathbf{U}^{\top} \mathbf{x}$ and $\mathbf{x}=\mathbf{U} \widehat{\mathbf{x}}$. Thus, the convolutional $*_G$ between the signal $\mathbf{x}$ and convolution kernel $\mathcal{F}$ can be defined as follows
\begin{equation}
\mathcal{F} *_{G} \mathbf{x}=\mathbf{U}\left(\left(\mathbf{U}^{\top} \mathcal{F} \right) \odot\left(\mathbf{U}^{\top} \mathbf{x}\right)\right)=\mathbf{U} \mathbf{g}_{\theta} \mathbf{U}^{\top} \mathbf{x}
\end{equation}
\noindent where $\odot$ denotes the element-wise product and $\mathbf{g}_{\theta}$ is a parameterized diagonal matrix. Most of the existing GNNs architectures can be regarded as a special instantiation on the convolutional kernel $\mathcal{F}$ (i.e., the matrix $\mathbf{g}_{\theta}$). For example, GCN-Cheby parameterizes $g_{\theta}$ with a polynomial expansion $\mathbf{g}_{\theta}\!=\!\sum_{k=0}^{K-1} \alpha_{k} \Lambda^{k}$, and GCN defines the convolutional kernel as $\mathbf{g}_{\theta}\!=\!\mathbf{I}_N-\Lambda$. Considering a special convolution kernel $\mathcal{F}_A\!=\!\mathbf{I}_N$, we have $\mathcal{F}_A *_{G} \mathbf{x} \!=\! \mathbf{U} \mathbf{I}_N \mathbf{U}^{\top} \mathbf{x} \!=\! \mathbf{U} \widehat{\mathbf{x}} \!=\! \mathbf{x}$, i.e., this is an identity mapping, where all information can be preserved. Next, we decompose the graph knowledge into low-frequency and high-frequency components \cite{bo2021beyond,wu2019simplifying} by factorizing $\mathcal{F}_A\!=\!\mathbf{I}_N$ as follows
\begin{equation*}
\begin{small}
\begin{aligned}
\mathcal{F}_A \!=\! \mathbf{I}_N  \!=\! \frac{1}{2}\bigg(\big(\underbrace{\mathbf{I}_N + \widetilde{\mathbf{D}}^{-\frac{1}{2}} \widetilde{\mathbf{A}} \widetilde{\mathbf{D}}^{-\frac{1}{2}}}_{\text{Low-Pass Filter} \  \mathcal{F}_M}\big) + \big(\underbrace{\mathbf{I}_N - \widetilde{\mathbf{D}}^{-\frac{1}{2}} \widetilde{\mathbf{A}} \widetilde{\mathbf{D}}^{-\frac{1}{2}}}_{\text{High-pass Filter} \  \mathcal{F}_H}\big)\bigg)
\end{aligned}
\end{small}
\end{equation*}
For a given signal $\mathbf{x}\in \mathbb{R}^{d}$, e.g., node feature, we have $\mathcal{F}_A *_{G} \mathbf{x} = \frac{1}{2} \left(\mathcal{F}_M + \mathcal{F}_H\right) *_{G} \mathbf{x} = \frac{1}{2} (\mathcal{F}_M *_{G} \mathbf{x} + \mathcal{F}_H *_{G} \mathbf{x}) = \mathbf{x}$, which means that any signal $\mathbf{x}$ can be decomposed into the average of two components $\mathcal{F}_M *_{G} \mathbf{x}$ and $\mathcal{F}_H *_{G} \mathbf{x}$. 
\newline
\vspace{-0.3em}

\noindent\textbf{Analysis on the Spectral Domain.}
The Proposition \ref{theorem:1} states what the two components $\mathcal{F}_M *_{G} \mathbf{x}$ and $\mathcal{F}_H *_{G} \mathbf{x}$ represent.
\begin{theorem} \label{theorem:1}
The convolution kernel $\mathcal{F}_M$ works as a low-pass filter, which filters out high-frequency information, and $\mathcal{F}_M *_{G} \mathbf{x}$ represents low-frequency knowledge; $\mathcal{F}_{H}$ works as a high-pass filter, which filters out low-frequency information, and $\mathcal{F}_H *_{G} \mathbf{x}$ represents high-frequency knowledge.
\end{theorem}
\begin{proof}
For a $L$-layer GNN, the signal $\mathbf{x}$ is filtered by the $L$-order convolution kernel $\mathcal{F}_M^L=(\mathbf{I}_N + \widetilde{\mathbf{D}}^{-\frac{1}{2}} \widetilde{\mathbf{A}} \widetilde{\mathbf{D}}^{-\frac{1}{2}})^L=(2\mathbf{I}_N - \mathbf{L})^L$ to output $\mathcal{F}_M^L *_{G} \mathbf{x}=\mathbf{U}(2 \mathbf{I}_N-\Lambda)^L \mathbf{U}^{\top} \mathbf{x}$ with $g_\theta^L(\lambda_i)=(2-\lambda_i)^L$. As shown in Fig.~\ref{fig:1}, $g_\theta^L(\lambda_i)$ decreases monotonically in the range $\lambda_i\in[0, 2]$ and reaches $g_\theta^L(\lambda_i\!=\!2)=0$ at $\lambda_i=2$, which mainly amplifies the low-frequency information and filters out the high-frequency information. Similarly, the $L$-order convolution kernel $\mathcal{F}_{H}^L$ has $g_\theta^L(\lambda_i)=\lambda_i^L$. As shown in Fig.~\ref{fig:1}, $g_\theta^L(\lambda_i)$ increases monotonically in the of range $\lambda_i\in[0, 2]$ and reaches $g_\theta^L(\lambda_i=0)=0$ at $\lambda_i=0$, which mainly filters out the low-frequency information but in turn amplifies the high-frequency information.
\end{proof}

\begin{figure}[!htbp]
    \vspace{-1em}
	\begin{center}
		\includegraphics[width=0.8\linewidth]{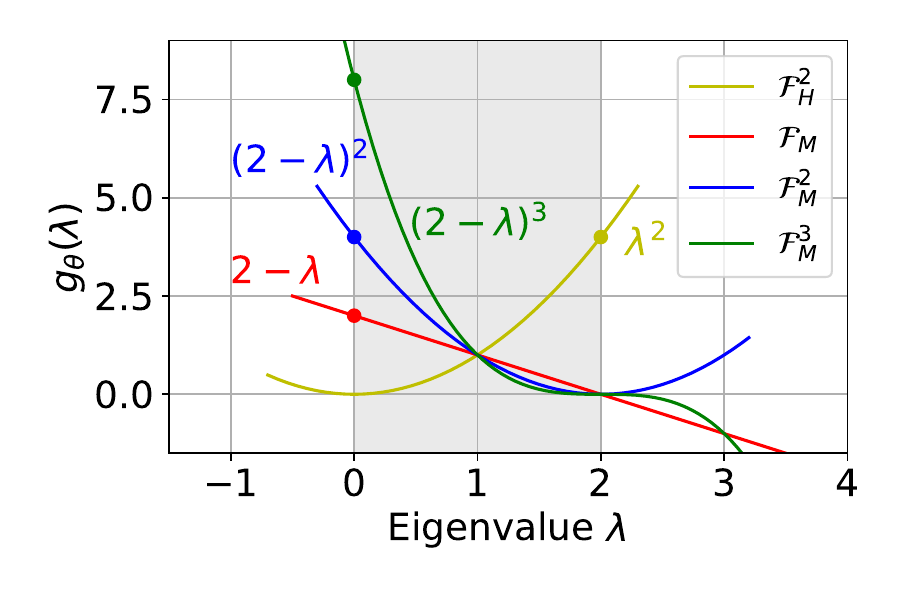}
	\end{center}
	\vspace{-2em}
	\caption{Eigenvalues \textit{vs.} Amplitudes}
	\vspace{-0.5em}
	\label{fig:1}
\end{figure}

\noindent\textbf{Correspondence on the Spatial Domain.} We have derived that $\mathcal{F}_M *_{G} \mathbf{x}$ and $\mathcal{F}_H *_{G} \mathbf{x}$ represent mainly the low- and high-frequency components of signal $\mathbf{x}$, and we next can derived their correspondences in the spatial domain, as follows
\vspace{-0.5em}
\begin{equation}
\begin{aligned}
    \mathcal{F}_M *_{G} \mathbf{x}_{i} \ \  \rightarrow \ \ \mathbf{x}_i^{(low)}  = \mathbf{x}_{i} + \sum_{j \in \mathcal{N}_{i}} \frac{\mathbf{x}_{j}}{\sqrt{|\mathcal{N}_{i}||\mathcal{N}_{j}|}} \\
    \mathcal{F}_H *_{G} \mathbf{x}_{i} \ \  \rightarrow \ \  \mathbf{x}_i^{(high)}  = \mathbf{x}_{i} - \sum_{j \in \mathcal{N}_{i}} \frac{\mathbf{x}_{j}}{\sqrt{|\mathcal{N}_{i}||\mathcal{N}_{j}|}} 
\end{aligned}
\label{equ:5}
\end{equation}
By the derivation in Eq.~(\ref{equ:5}), the low-frequency knowledge $\mathcal{F}_M*_G \mathbf{x}$ is the sum of node feature and its neighborhood features in the spatial domain. On the other hand, the high-frequency knowledge $\mathcal{F}_{H}*_G \mathbf{x}$ represents the differences between the target node feature with its neighborhood features. There have recently been some novel GNN models \cite{fagcn2021,pei2020geom,zhu2020beyond,chien2021adaptive} that can capture both low- and high-frequency information simultaneously or adaptively. However, in this paper, we focus on the design of distillation objective functions and do not consider indirect performance improvements by employing these more powerful but complex GNNs. Instead, we consider the most commonly used GNNs, such as GCN \cite{kipf2016semi}, GraphSAGE \cite{hamilton2017inductive}, and GAT \cite{velivckovic2017graph}, all of which rely on multi-layer message passing to aggregate features of neighboring nodes that are multiple hops away, i.e., they essentially work as a low-pass filter $\mathcal{F}_M^L$ or its variants.

\section{Roles Played by Low- and High-Frequency Knowledge during Distillation}
\subsection{Rethinking the Core of Knowledge Distillation}
We rethink the core of knowledge distillation from three shallow-to-deep perspectives to highlight our motivations. 

\begin{itemize}
    \item \textit{\textbf{Firstly}}, knowledge distillation enables the representations of MLPs to ``mimic" those of GNNs as closely as possible by imposing KL-divergence constraints between their softmax distribution probabilities. However, such a mimicking (or fitting) process is inevitably accompanied by a loss of information, especially high-frequency information, which explains why the performance of student MLPs is always hard to match with that of teacher GNNs.
    \item \textit{\textbf{Secondly}}, for a neural network framework, any change in the final representations is achieved indirectly by optimizing the mapping function, i.e., the network parameters. In this sense, knowledge distillation essentially optimizes the parameter matrices $\{\mathbf{W}^{(l)}\}_{l=0}^{L-1}$ of the student MLPs to make it functionally approximates the convolution kernel of the teacher GNNs, which makes the student MLPs also serve as a low-pass filter $\widetilde{\mathcal{F}}_M^L$ for graph data.
    \item \textit{\textbf{Finally}}, the low-pass filter in the spectral domain is equivalent to neighborhood aggregation in the spatial domain as derived in Eq.~(\ref{equ:5}), which in essence can be considered as a special use of the graph topology.
\end{itemize}  

\begin{figure*}[!tbp]
    \vspace{-1.5em}
	\begin{center}
		\subfigure[Mean Cosine Similarity]{\includegraphics[width=0.245\linewidth]{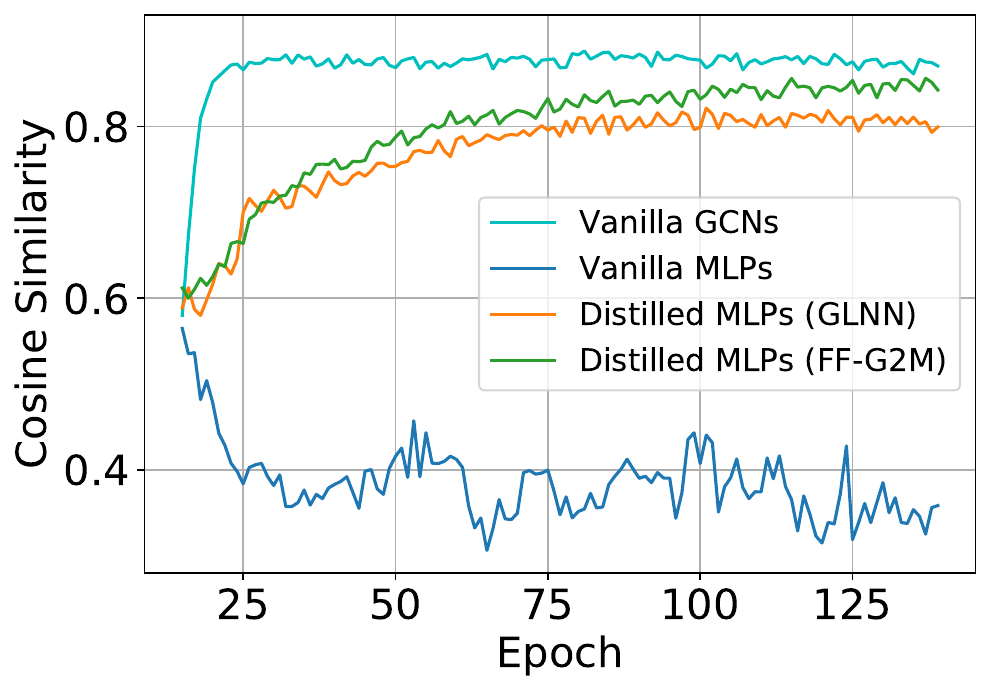}\label{fig:2a}}
	    \subfigure[Pairwise Distance Differences]{\includegraphics[width=0.245\linewidth]{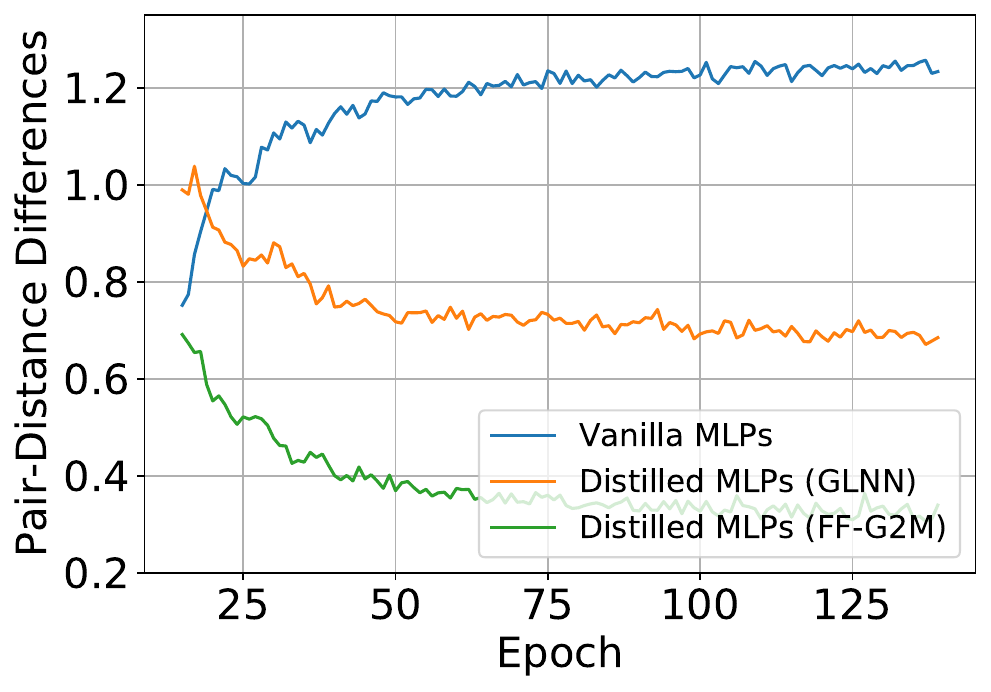}\label{fig:2b}}
	    \subfigure[Spectral Analysis]{\includegraphics[width=0.245\linewidth]{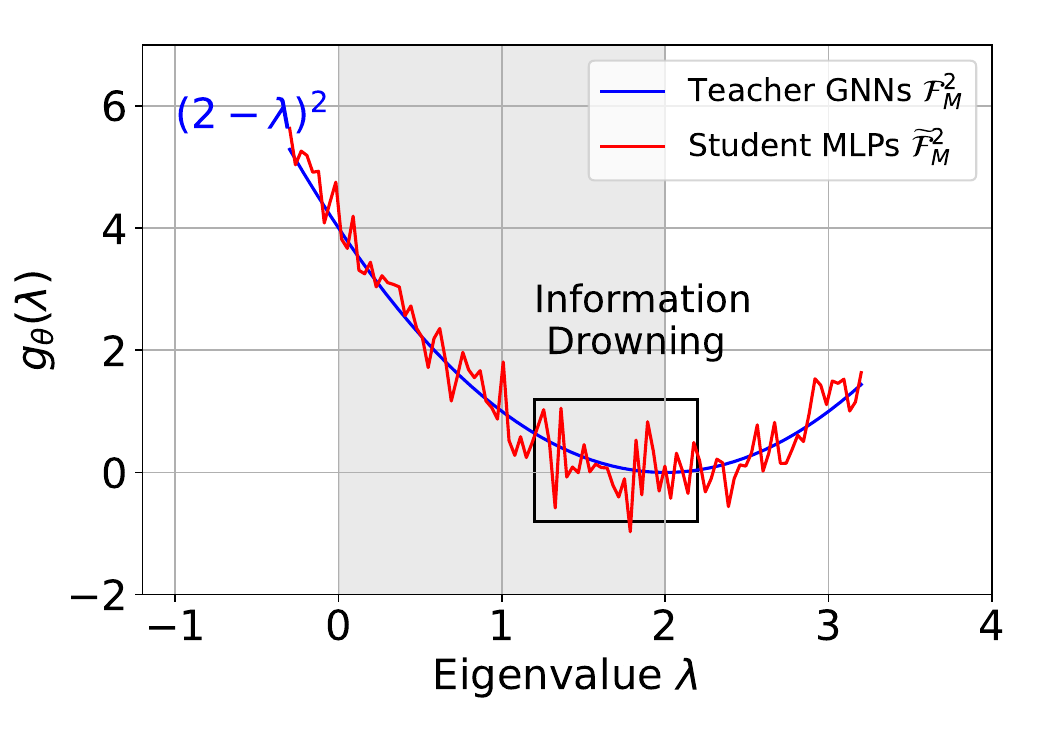}\label{fig:2c}}
	    \subfigure[Spatial Analysis]{\includegraphics[width=0.235\linewidth]{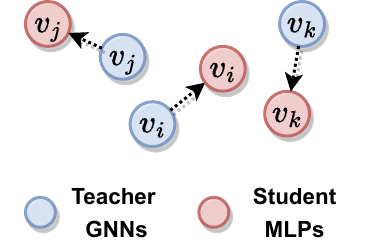}\label{fig:2d}}
	\end{center}
	\vspace{-0.8em}
	\caption{(a) Mean cosine similarity (the higher, the better) between nodes with their first-order neighbors on Cora. (b) Pairwise distance differences (the lower, the better) between teacher GCNs and student MLPs on Cora. (c)(d) Illustrations of how the high-frequency information drowning arises and what potential impact it has in the spectral and spatial domains, respectively.}
	\vspace{-1em}
	\label{fig:2}
\end{figure*}

To explore the roles played by graph topology during GNN-to-MLP distillation, we plot the \textit{mean cosine similarity} of nodes with their first-order neighbors for vanilla GCNs, vanilla MLPs, and Distilled MLPs (GLNN) on the Cora dataset in Fig.~\ref{fig:2a}, from which we observe that the mean similarity of GCNs and GLNN gradually increases with training, while that of vanilla MLPs gradually decreases, which indicates that knowledge distillation has introduced graph topology as an inductive bias (as GCNs has done), while vanilla MLPs do not. As a result, the distilled MLPs can enjoy the benefits of topology-awareness in training but without neighborhood-fetching latency in inference.

\subsection{High-Frequency Information Drowning}
Next, we discuss a potential high-frequency information drowning problem from both spectral and spatial domains, i.e., the high-frequency information of the pre-trained GNNs may be overwhelmed by the low-frequency knowledge during the process of GNN-to-MLP knowledge distillation. 
\newline
\vspace{-0.5em}

\noindent \textbf{How information drowning arises?} 
\textit{From the perspective of spectral domain}, the knowledge distillation optimizes the network parameters of the student MLPs to make it functionally approximate the convolution kernel of the teacher GNNs, i.e., $\widetilde{\mathcal{F}}_M^L \!\approx\! \mathcal{F}_M^L$. The information loss induced by such approximation may be inconsequential for high-amplitude low-frequency information but can be catastrophic for those high-frequency information with very low amplitude, as shown in Fig.~\ref{fig:2c}. As a result, compared to low-frequency information, high-frequency information is more likely to be drowned by these optimization errors.
\newline
\vspace{-0.5em}

\noindent \textbf{What impact does information drowning have?}
\textit{From the perspective of spatial domain}, the information drowning may lead to distilled MLPs that, despite preserving neighborhood smoothing well, can easily neglect differences between nodes, such as pairwisde distances. To illustrate this, we consider a target node $v_i$ and its two neighboring nodes $v_j$ and $v_k$ in Fig.~\ref{fig:2d}, where they are mapped closely by GNNs. In the process of knowledge distillation, the representations of these three nodes may be mapped around the representations of teacher GNNs, i.e., they are still mapped closely with most of their low-frequency information preserved; however, the relative distances between nodes, i.e., their high-frequency information, may be drowned dramatically. For example, node $v_i$ is adjacent to node $v_j$ but far from node $v_k$ in the representation space of teacher GNNs. However, in the representation space of student MLPs, node $v_i$ becomes closer to node $v_k$ and farther from node $v_j$.

The curves of the pairwise distance differences between the teacher GCNs and the student MLP in Fig.~\ref{fig:2b} show that common knowledge distillation (e.g., GLNN) is not good at capturing high frequency information, compared to our proposed FF-G2M. Moreover, extensive qualitative and quantitative experiments have been provided to demonstrate the harmfulness of the identified high-frequency information drowning problem in the experimental section. The detailed experimental settings, including hyperparameters and evaluation metric definitions, are available in \textbf{Appendix B\&E}.

\begin{figure*}[!tbp]
	\begin{center}
		\includegraphics[width=1.0\linewidth]{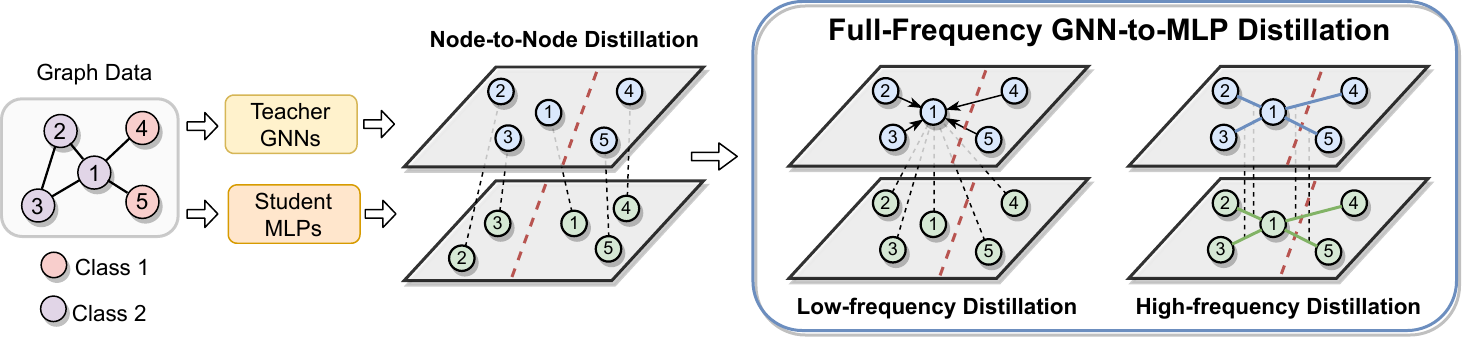}
	\end{center}
	\vspace{-0.5em}
	\caption{Illustration of the Full-Frequency GNN-to-MLP (FF-G2M) distillation framework, where the dotted red lines denote the predicted class-boundary, the solid black lines denote feature aggregation from the neighborhood, and the dashed black lines denote the distillation of knowledge (neighborhood features and pairwise distances) from teacher GNNs to student MLPs.}
	\vspace{-1em}
	\label{fig:3}
\end{figure*}

\section{Full-Frequency GNN-to-MLP (FF-G2M) Knowledge Distillation}
The above discussions reached two important insights: (1) the inductive bias of graph topology plays an important role, and (2) it is mainly the low-frequency knowledge of graph data that has been distilled from the teacher GNNs to the student MLPs. Inspired by these two insights, we propose \textit{Low-Frequency Distillation} (LFD) and \textit{Hign-Frequency Distillation} (HFD) to fully capture the low-frequency and high-frequency knowledge learned by GNNs, respectively. An high-level overview of the proposed \textit{Full-Frequency GNN-to-MLP} (FF-G2M) framework is shown in Fig.~\ref{fig:3}.

\subsection{Low-Frequency Distillation (LFD)} The node representations of teacher GNNs are generated by explicit message passing, so it mainly captures the low-frequency information of the graph data as analyzed earlier. Unlike aggregating features from neighborhoods as in GNNs, we directly distill (diffuse) knowledge from teacher GNNs into the neighborhoods of student MLPs in order to better utilize topological information and low-frequency knowledge captured by GNNs, formulated as follows
\begin{equation}
\hspace{-0.4em}
\mathcal{L}_{\mathrm{LFD}}\!=\!\frac{1}{|\mathcal{E}|}\sum_{i\in\mathcal{V}}\sum_{j\in\mathcal{N}_i \cup i }\!\mathcal{D}_{KL}\Big(\sigma(\mathbf{z}_{j}^{(L)} / \tau_1), \sigma(\mathbf{h}_{i}^{(L)} / \tau_1)\Big)
\label{equ:6}
\end{equation}
where $\tau_1$ is the low-frequency distillation temperature, and $\sigma=\operatorname{softmax}(\cdot)$ denotes an activation function.

\subsection{High-Frequency Distillation (HFD)}
As derived in Eq.~(\ref{equ:5}), the high-frequency components in the spectral domain represent the differences between node feature and its neighborhood features in the spatial domain. Inspired by this, we propose High-Frequency Distillation (HFD), a GNN knowledge objective that trains student MLPs to preserve the neighborhood pairwise differences from the representation space of teacher GNNs. The neighborhood pairwise differences around node $v_i$ are defined as the differences between the target node feature $\mathbf{s}_{i}$ and its
neighborhood features $\{\mathbf{s}_{j} \ |\  j\in\mathcal{N}_i\}$, which can be computed by the kernel $\mathcal{K}\left(\mathbf{s}_{i}, \mathbf{s}_{j}\right) = \left|\mathbf{s}_{i}-\mathbf{s}_{j}\right|$, where $|\cdot|$ denotes the element-wise absolute values. The high-frequency distillation trains the student model
to mimic the neighborhood pairwise differences from the teacher GNNs via KL-divergence constraints, which can be defined as follows
\begin{equation}
\begin{aligned}
\mathcal{L}_{\mathrm{HFD}}\!=\!\frac{1}{|\mathcal{E}|}\sum_{i\in\mathcal{V}}\sum_{j \in \mathcal{N}_i}\mathcal{D}_{\mathrm{KL}}\Big(\sigma\big(\mathcal{K}\big(\mathbf{z}_{i}^{(L)}, \mathbf{z}_{j}^{(L)}\big)/\tau_2\big), \\ \sigma\big(\mathcal{K}\big(\mathbf{h}_{i}^{(L)}, \mathbf{h}_j^{(L)}\big)/\tau_2\big)\Big)
\end{aligned}
\label{equ:7}
\end{equation}
where $\mathcal{K}\left(\cdot, \cdot\right)$ denotes the element-wise absolute values, and $\tau_2$ is the high-frequency distillation temperature.

\subsection{Training Strategy} 
The pseudo-code of the FF-G2M framework is summarized in \textbf{Appendix C}. To achieve GNN-to-MLP knowledge distillation, we first pre-train the teacher GNNs with the classification loss  $\mathcal{L}_{\mathrm{label}}=\frac{1}{|\mathcal{V}_L|}\sum_{i\in\mathcal{V}_L}\mathcal{H}\big(y_i, \sigma(\mathbf{h}_i^{(L)})\big)$, where $\mathcal{H}(\cdot)$ denotes the cross-entropy loss and $y_i$ is the ground-truth label of node $v_i$. Finally, the total objective function to distill the low- and high-frequency knowledge from the teacher GNNs into the student MLPs is defined as follows
\begin{equation*}
\mathcal{L}_{\mathrm{total}} =  \frac{\lambda}{|\mathcal{V}_L|}\sum_{i\in\mathcal{V}_L}\mathcal{H}\big(y_i, \sigma(\mathbf{z}_i^{(L)})\big) + \big(1-\lambda\big) \big(\mathcal{L}_{\mathrm{LFD}} + \mathcal{L}_{\mathrm{HFD}}\big)
\end{equation*}
where $\lambda$ is the weights to balance the influence of the classification loss and two knowledge distillation losses. The time complexity analysis of FF-G2M is available in \textbf{Appendix D}.

\subsection{Discussion and Comparision} 
In this subsection, we compare the proposed FF-G2M framework with the commonly used node-to-node distillation (e.g., GLNN) in Fig.~\ref{fig:3}. While the node-to-node distillation can map neighboring nodes closely in the representation space of MLPs, i.e., preserving low-frequency knowledge, it completely confounds the relative distance between node pairs, i.e., high-frequency knowledge is drowned, leading to a different (incorrect) class-boundary with the teacher GNNs. In terms of the proposed FF-G2M framework, the Low-Frequency Distillation distills (diffuses) the features aggregated from the neighborhood in the teacher GNNs back into their neighborhood of the student MLPs to better utilize the extracted low-frequency knowledge. Besides, High-Frequency Distillation directly distills the neighborhood pairwise differences from teacher GNNs into the student MLPs to better capture the high-frequency knowledge patterns, i.e., the relative positions between pairs of nodes.

\begin{table*}[!htbp]
\begin{center}
\vspace{-1em}
\caption{Classificatiom accuracy  $\pm$  std (\%) on six real-world datasets, where we consider three different GNN architectures (GCN, GraphSAGE, and GAT) as the teacher and pure MLPs as the student model. The best metrics are marked by \textbf{bold}.}
\vspace{-0.5em}
\label{tab:1}
\resizebox{0.85\textwidth}{!}{
\begin{tabular}{clccccccc}

\toprule
\textbf{Teacher GNNs} & \multicolumn{1}{c}{\textbf{Method}} & \textbf{Cora} & \textbf{Citeseer} & \textbf{Pubmed} & \textbf{Amazon-Photo} & \textbf{Coauthor-CS} & \textbf{Coauthor-Phy} & \textbf{\textit{Average}}$\uparrow$ \\ \midrule
\multirow{8}{*}{GCN} & Vanilla GCN & 82.2 $\pm$ 0.5 & 71.6 $\pm$ 0.4 & 79.3 $\pm$ 0.3 & 91.8 $\pm$ 0.6 & 89.9 $\pm$ 0.7 & 91.9 $\pm$ 1.2 & - \\
 & Vanilla MLP & 59.7 $\pm$ 1.0 & 60.7 $\pm$ 0.5 & 71.5 $\pm$ 0.5 & 77.4 $\pm$ 1.2 & 87.5 $\pm$ 1.4 & 89.2 $\pm$ 0.9 & - \\ 
 & GLNN & 82.8 $\pm$ 0.5 & 72.7 $\pm$ 0.4 & 80.2 $\pm$ 0.6 & 91.4 $\pm$ 1.0 & 92.7 $\pm$ 1.0 & 93.2 $\pm$ 0.5 & - \\ \cmidrule(r){2-9} 
 & FF-G2M & \textbf{84.3 $\pm$ 0.4} & \textbf{74.0 $\pm$ 0.5} & \textbf{81.8 $\pm$ 0.4} & \textbf{94.2 $\pm$ 0.4} & \textbf{93.8 $\pm$ 0.5} & \textbf{94.4 $\pm$ 0.9} & - \\
 & $\Delta_{GCN}$ & 2.1 & 2.4 & 2.5 & 2.4 & 3.9 & 2.5 & 2.63 \\
 & $\Delta_{MLP}$ & 24.6 & 13.3 & 10.3 & 16.8 & 6.3 & 5.2 & 12.75 \\
 & $\Delta_{GLNN}$ & 1.5 & 1.3 & 1.6 & 2.8 & 1.1 & 1.2 & 1.58 \\ \midrule
\multirow{8}{*}{GraphSAGE} & Vanilla SAGE & 82.5 $\pm$ 0.6 & 70.9 $\pm$ 0.6 & 77.9 $\pm$ 0.4 & 92.0. $\pm$ 0.6 & 89.7 $\pm$ 1.0 & 92.2 $\pm$ 0.9 &  - \\
 & Vanilla MLP & 59.7 $\pm$ 1.0 & 60.7 $\pm$ 0.5 & 71.5 $\pm$ 0.5 & 77.4 $\pm$ 1.2 & 87.5 $\pm$ 1.4 & 89.2 $\pm$ 0.9 & - \\
 & GLNN & 82.7 $\pm$ 0.8 & 70.5 $\pm$ 0.5 & 79.7 $\pm$ 0.6 & 91.0 $\pm$ 0.9 & 92.5 $\pm$ 1.0 & 93.0 $\pm$ 0.7 & - \\ \cmidrule(r){2-9}
 & FF-G2M & \textbf{84.5 $\pm$ 0.8} & \textbf{73.1 $\pm$ 1.2} & \textbf{80.9 $\pm$ 0.5} & \textbf{94.1 $\pm$ 0.5} & \textbf{93.6 $\pm$ 0.5} & \textbf{94.5 $\pm$ 1.1} & - \\
 & $\Delta_{SAGE}$ & 2.0 & 2.2 & 3.0 & 2.1 & 3.9 & 2.3 & 2.58 \\
 & $\Delta_{MLP}$ & 24.8 & 12.4 & 9.4 & 16.7 & 6.1 & 5.3 & 12.45 \\
 & $\Delta_{GLNN}$ & 1.8 & 2.6 & 1.2 & 3.1 & 1.1 & 1.5 & 1.88 \\ \midrule
\multirow{8}{*}{GAT} & Vanilla GAT & 81.9 $\pm$ 1.0 & 71.2 $\pm$ 0.5 & 78.6 $\pm$ 0.4 & 91.4 $\pm$ 0.8 & 90.5 $\pm$ 0.8 & 92.3 $\pm$ 1.5 & - \\
 & Vanilla MLP & 59.7 $\pm$ 1.0 & 60.7 $\pm$ 0.5 & 71.5 $\pm$ 0.5 & 77.4 $\pm$ 1.2 & 87.5 $\pm$ 1.4 & 89.2 $\pm$ 0.9 & - \\
 & GLNN & 82.1 $\pm$ 0.7 & 70.6 $\pm$ 0.8 & 80.4 $\pm$ 1.0 & 91.7 $\pm$ 0.7 & 92.7 $\pm$ 0.9 & 92.7 $\pm$ 0.6 & - \\ \cmidrule(r){2-9}
 & FF-G2M & \textbf{84.0 $\pm$ 0.7} & \textbf{73.5 $\pm$ 0.8} & \textbf{81.1 $\pm$ 0.6} & \textbf{93.9 $\pm$ 0.7} & \textbf{94.0 $\pm$ 0.7} & \textbf{94.7 $\pm$ 1.2} & - \\
 & $\Delta_{GAT}$ & 2.1 & 2.3 & 2.5 & 2.5 & 3.5 & 2.4 & 2.55 \\
 & $\Delta_{MLP}$ & 24.3 & 12.8 & 9.6 & 16.5 & 6.5 & 5.5 & 12.53 \\
 & $\Delta_{GLNN}$ & 1.9 & 2.9 & 0.7 & 2.2 & 1.3 & 2.0 & 1.83 \\ \bottomrule
 
\end{tabular}} \vspace{-1em}
\end{center}
\end{table*}

\section{Experiments} \label{sec:6}
\textbf{Datasets.}
The effectiveness of the FF-G2M framework is evaluated on six public real-world datasets, including Cora \cite{sen2008collective}, Citeseer \cite{giles1998citeseer}, Pubmed \cite{mccallum2000automating}, Coauthor-CS, Coauthor-Physics, and Amazon-Photo \cite{shchur2018pitfalls}. For each dataset, following the data splitting settings of \cite{kipf2016semi,liu2020towards}, we select 20 nodes per class to construct a training set, 500 nodes for validation, and 1000 nodes for testing.  A statistical overview of these datasets is placed in \textbf{Appendix A}. Besides, we defer the implementation details and hyperparameter settings for each dataset to \textbf{Appendix B} and supplementary materials.
\newline
\vspace{-0.5em}

\noindent \textbf{Baselines.}
Three basic components in knowledge distillation are (1) teacher model, (2) student model, and (3) distillation loss. As a model-agnostic general framework, FF-G2M can be combined with any teacher GNN architecture. In this paper, we consider three types of teacher GNNs, including GCN \cite{kipf2016semi}, GraphSAGE \cite{hamilton2017inductive}, and GAT \cite{velivckovic2017graph}. As for the student model, we default to using pure MLPs (with the same network architecture as the teacher GNNs) as the student model for a fair comparison. Finally, the focus of this paper is on designing distillation objectives rather than powerful teacher and student models. Therefore, we only take GLNN \cite{zhang2021graph} as an important baseline to compare FF-G2M with the conventional node-to-node distillation approach. The experiments of all baselines and FF-G2M are implemented based on the standard implementation in the DGL library \cite{wang2019dgl} using PyTorch 1.6.0 library on NVIDIA V100 GPU. Each set of experiments is run five times with different random seeds, and the average are reported as metrics.

\vspace{-0.5em}
\subsection{Classification Performance Comparison}
This paper aims to explore which knowledge patterns \textit{should} and \textit{how to} be distilled into student MLPs, rather than designing more powerful teacher GNNs. Therefore, we consider three classical GNNs, including GCN, GraphSAGE, and GAT, as the teacher models and distill their knowledge into MLPs with the same network architecture. The experimental results on six datasets are reported in Table.~\ref{tab:1}, from which we can make the following observations: (1) In general, more powerful teacher GNNs can lead to student MLPs with better classification performance. However, such improvements are usually very limited and do not work for all datasets and GNN architectures. For example, on the Citeseer dataset, the performance of GLNN drops over the vanilla implementation of teacher GNNs by 0.4\% (GraphSAGE) and 0.6\% (GAT), respectively. (2) The proposed FF-G2M framework can consistently improve the performance of student MLPs across three GNN architectures on all six datasets. For example, FF-G2M can outperform vanilla teacher GNNs by 2.63\% (GCN), 2.58\% (GraphSAGE), and 2.55\% (GAT) averaged over six datasets, respectively.

\begin{figure*}[!tbp]
    \vspace{-0.5em}
	\begin{center}
		\subfigure[Pre-trained GCNs]{\includegraphics[width=0.24\linewidth]{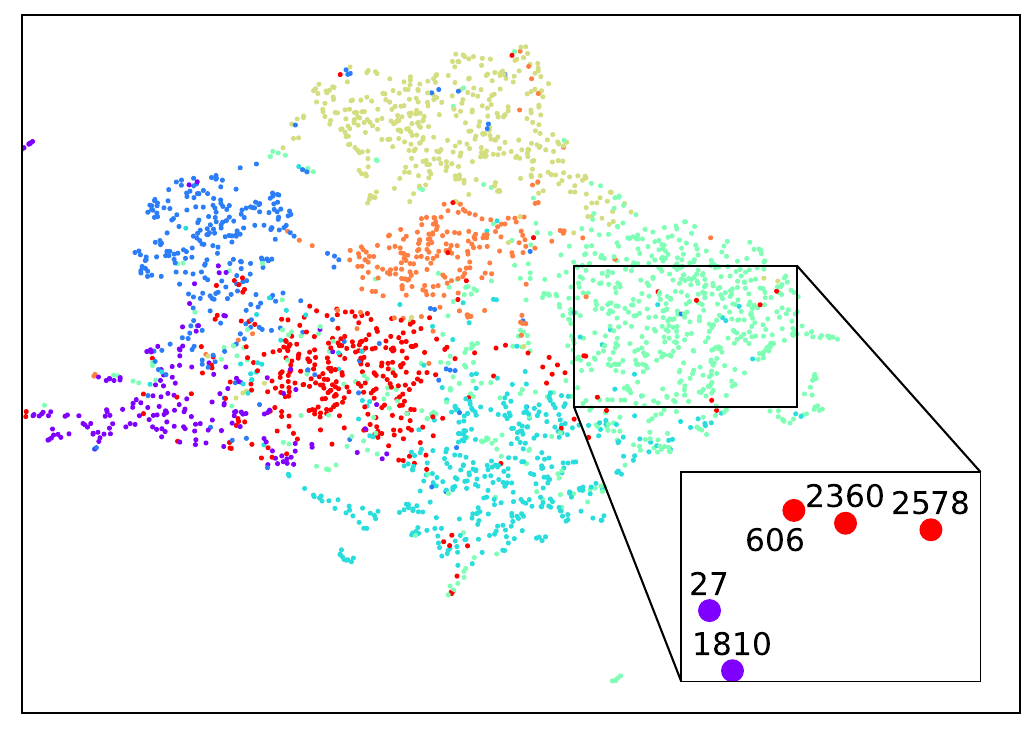}\label{fig:4a}}
		\subfigure[Vanilla MLPs]{\includegraphics[width=0.24\linewidth]{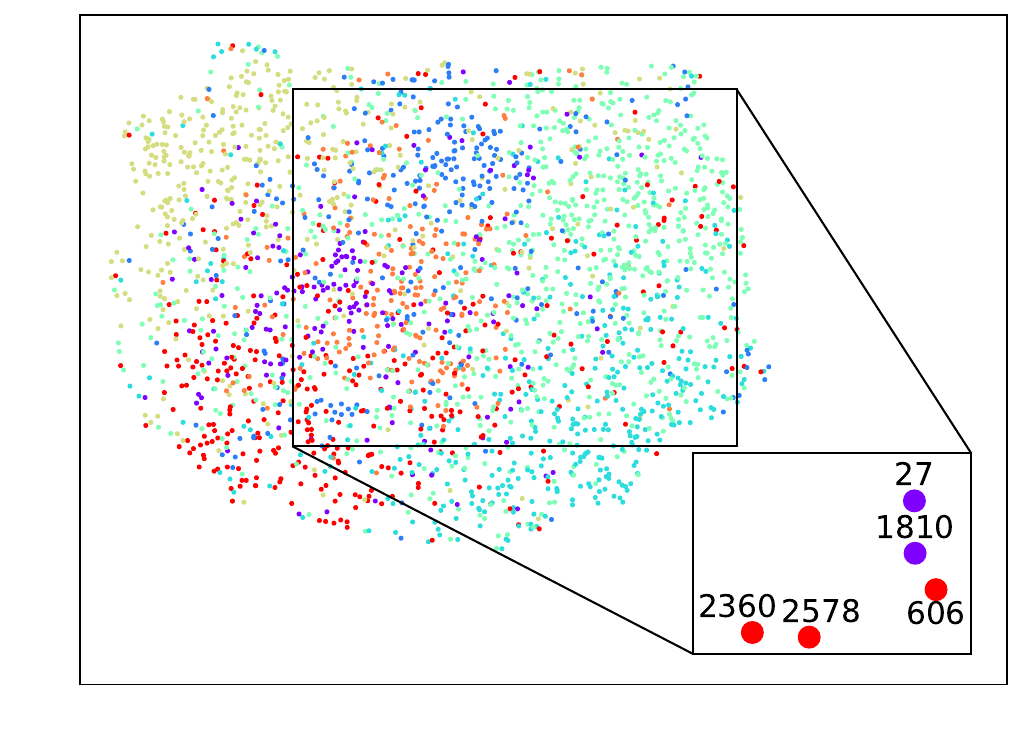}\label{fig:4b}}
		\subfigure[Distilled MLPs (GLNN)]{\includegraphics[width=0.24\linewidth]{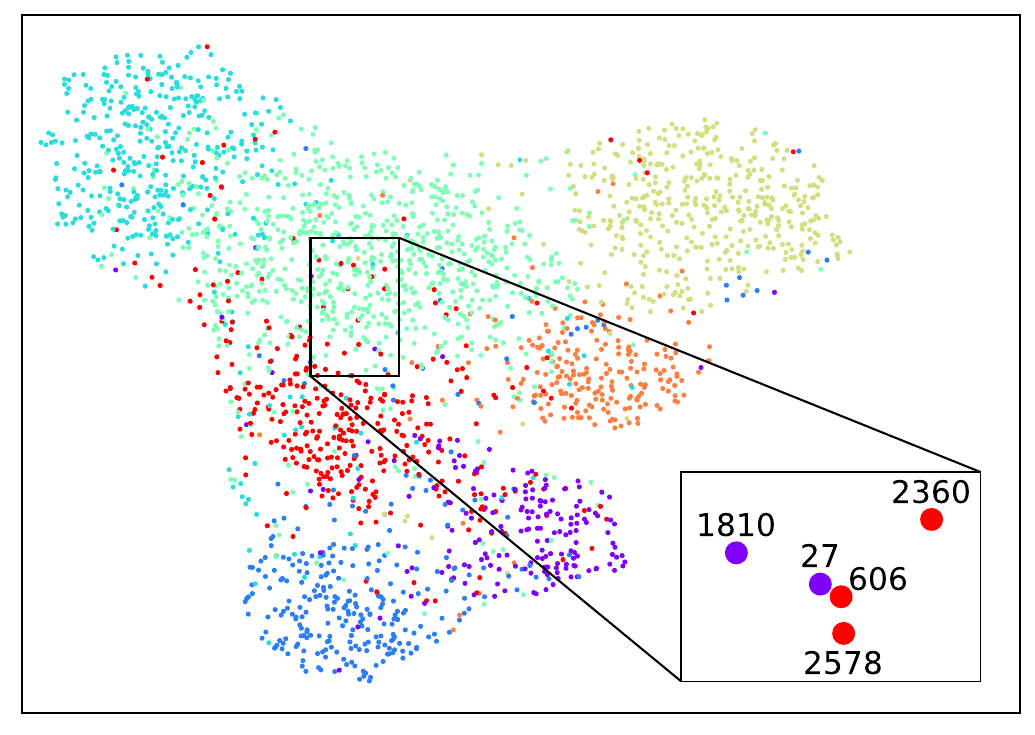}\label{fig:4c}}
		\subfigure[Distilled MLPs (FF-G2M)]{\includegraphics[width=0.24\linewidth]{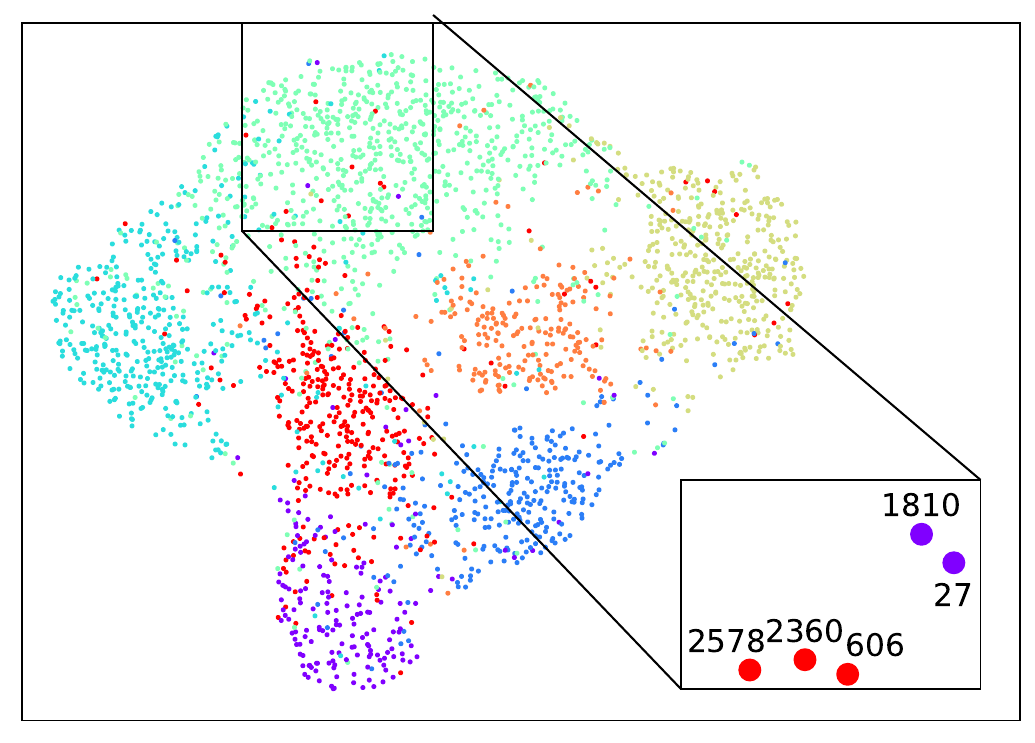}\label{fig:4d}}
	\end{center}
	\vspace{-1em}
	\caption{Representation 2D-Visualizations (by UMAP \cite{mcinnes2018umap}) of the teacher model and three student models on Cora. Each node is colored by its ground-truth label, and the numbers around the nodes denote the node ids.}
	\label{fig:4}
\end{figure*}

\begin{table*}[!tbp]
\begin{center}
\caption{Classificatiom accuracy $\pm$ std (\%) on six real-world datasets. The best metrics are marked by \textbf{bold}.}
\vspace{-0.5em}
\label{tab:2}
\resizebox{0.85\textwidth}{!}{
\begin{tabular}{lcccccc}

\toprule
\textbf{Method} & \textbf{Cora} & \textbf{Citeseer} & \textbf{Pubmed} & \textbf{Amazon-Photo} & \textbf{Coauthor-CS} & \textbf{Coauthor-Phy} \\ \midrule
Vanilla GCN & 82.2 $\pm$ 0.5 & 71.6 $\pm$ 0.4 & 79.3 $\pm$ 0.3 & 91.8 $\pm$ 0.6 & 89.9 $\pm$ 0.7 & 91.9 $\pm$ 1.2 \\
Vanilla MLP & 59.7 $\pm$ 1.0 & 60.7 $\pm$ 0.5 & 71.5 $\pm$ 0.5 & 77.4 $\pm$ 1.2 & 87.5 $\pm$ 1.4 & 89.2 $\pm$ 0.9 \\ 
GLNN \cite{zhang2021graph} & 82.8 $\pm$ 0.5 & 72.7 $\pm$ 0.4 & 80.2 $\pm$ 0.6 & 91.4 $\pm$ 1.0 & 92.7 $\pm$ 1.0 & 93.2 $\pm$ 0.5 \\ \midrule \midrule
Low-Frequency KD w/ $\mathcal{L}_{\mathrm{LFD}}$ & 83.4 $\pm$ 0.9 & 73.7 $\pm$ 0.6 & 81.0 $\pm$ 0.5 & 92.1 $\pm$ 0.8 & 93.2 $\pm$ 0.8 & 93.7 $\pm$ 0.8 \\
High-Frequency KD w/ $\mathcal{L}_{\mathrm{HFD}}$ & 68.5 $\pm$ 0.8 & 63.2 $\pm$ 0.7 & 74.4 $\pm$ 0.4 & 82.5 $\pm$ 1.3 & 89.3 $\pm$ 1.7 & 91.0 $\pm$ 1.6 \\
FF-G2M (full model) & \textbf{84.3 $\pm$ 0.4} & \textbf{74.0 $\pm$ 0.5} & \textbf{81.8 $\pm$ 0.4} & \textbf{94.2 $\pm$ 0.4} & \textbf{93.8 $\pm$ 0.5} & \textbf{94.4 $\pm$ 0.9} \\ \toprule

\end{tabular}} \vspace{-1.5em}
\end{center}
\end{table*}

\vspace{-0.5em}
\subsection{Qualitative and Quantitative Analysis}
Extensive qualitative and quantitative experiments are conducted to explore the existence and harmfulness of the information drowning problem and how to solve it by FF-G2M.

\subsubsection{Qualitative Analysis on Visualizations.} We consider GCNs as the teacher model and compare its visualization with that of vanilla MLPs, GLNN, and FF-G2M on the Cora dataset (due to space limitations, more results can be found in \textbf{Appendix F}). We select a target node (id 27 for Cora) and analyze its relative position relationship with its neighbors in Fig.\ref{fig:4}, from which we observe that: (1) The vanilla MLPs map neighboring nodes apart, which indicates that it is not even good at capturing low-frequency information. (2) GLNN fails to capture the relative positions between the target node and its neighbors, i.e., high-frequency information. (3) FF-G2M well preserves the relative positions between nodes while mapping neighboring nodes closely, which suggests that it is good at capturing both low- and high-frequency information. For example, on the Cora dataset, the target node (id 27) is the closest to node (id 1810) and the farthest from node (id 2678) in the visualizations of both teacher GCNs and FF-G2M's student MLPs.

\subsubsection{Quantitative Analysis on Evaluation Metrics.} To study what knowledge patterns of GNNs are actually distilled into MLPs during GNN-to-MLP distillation, we consider both (1) low-frequency knowledge, measured by the \textit{mean cosine similarity} of nodes with their first-order neighbors, and (2) high-frequency knowledge, measured by \textit{KL-divergence} between the pairwise distances of teacher GNNs and student MLPs, respectively. The detailed mathematical definitions of these two evaluation metrics are available in \textbf{Appendix E}. From the experimental results on the Cora dataset reported in Fig.~\ref{fig:5}, we make four observations: (1) The vanilla MLP does not consider the inductive bias of the graph topology at all and thus fails to capture the low-and high-frequency knowledge in the graph data. (2) GLNN is capable of successfully capturing low-frequency information, i.e., neighborhood smoothing, but is not good at capturing high-frequency knowledge, i.e., difference information between pairs of nodes. (3) The proposed low- and high-frequency distillation has an advantage over GLNN in capturing one type of individual frequency but lags behind in another frequency. (4) The proposed FF-G2M combines both the two distillation and is better at capturing both low- and high-frequency knowledge than GLNN, especially the latter.

\begin{figure}[!htbp]
    \vspace{-0.5em}
	\begin{center}
		\subfigure[Low-frequency Knowledge]{\includegraphics[width=0.48\linewidth]{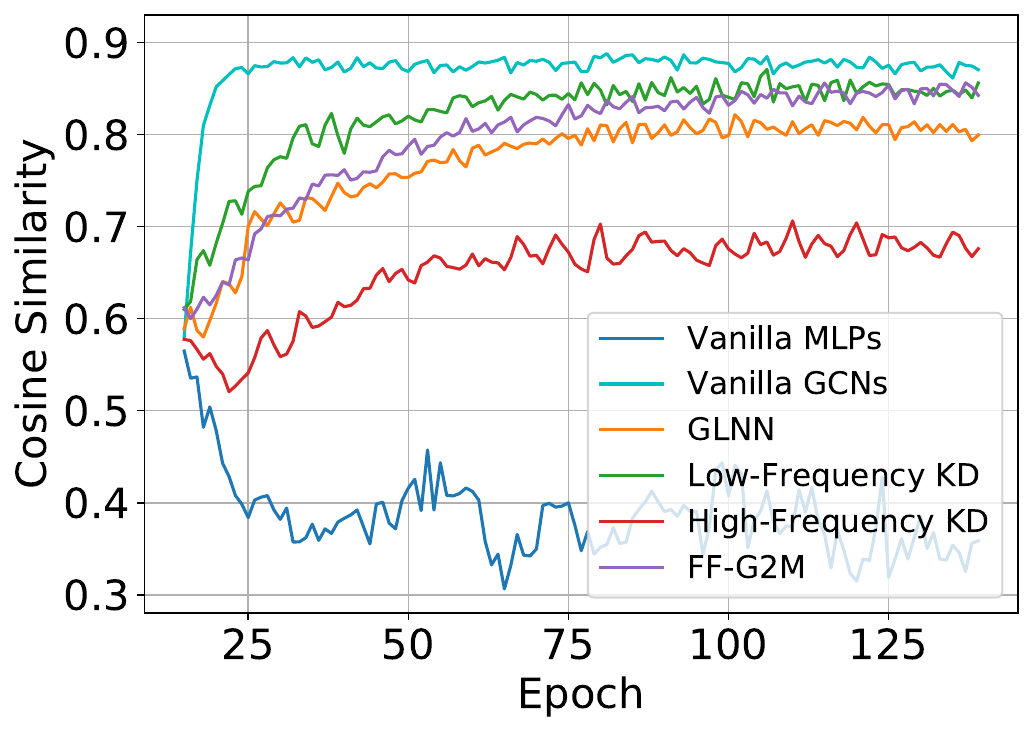}\label{fig:5a}}
        \subfigure[High-frequency Knowledge]{\includegraphics[width=0.48\linewidth]{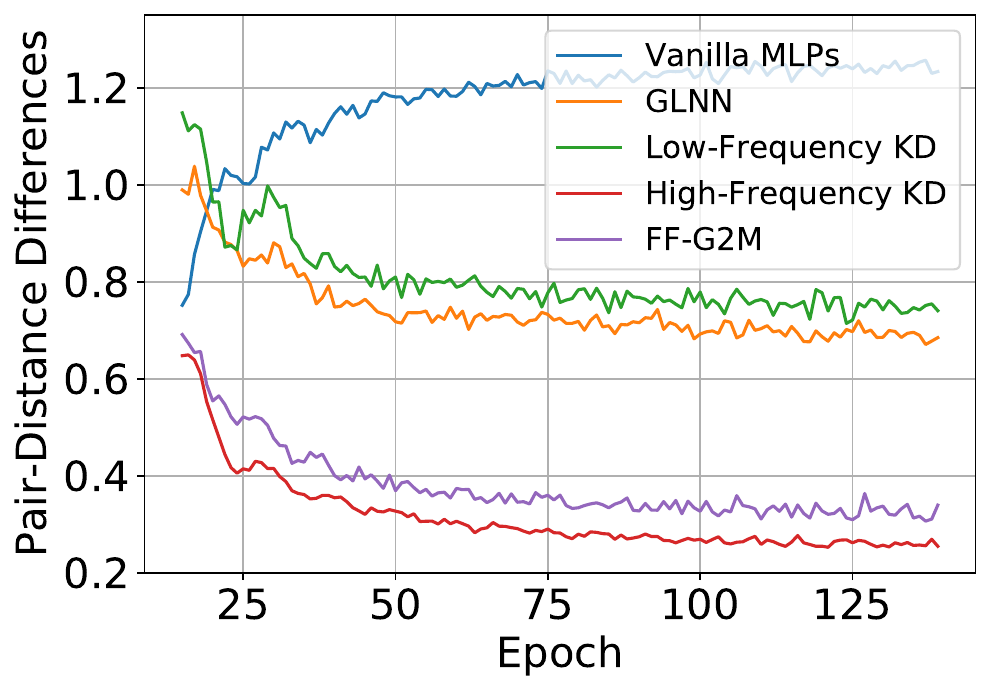}\label{fig:5b}}
	\end{center}
	\vspace{-1em}
	\caption{(a) Curves of mean cosine similarity (\textbf{the higher, the better}) between nodes with their first-order neighbors. (b) Curves of pairwise distance differences (\textbf{the lower, the better}) between the teacher GNNs and student MLPs.}
	\vspace{-1em}
	\label{fig:5}
\end{figure}

\subsection{Roles of Low- and High-frequency Knowledge} 
To evaluate the roles played by low- and high-frequency knowledge in GNN-to-MLP distillation, we consider distillation with only $\mathcal{L}_{\mathrm{LFD}}$ and $\mathcal{L}_{\mathrm{HFD}}$, in addition to the full FF-G2M model. The experiments (with GCNs as the teacher model) on six datasets are reported in Table.~\ref{tab:2}, from which we observe that: (1) The proposed low-frequency distillation $\mathcal{L}_{\mathrm{LFD}}$ makes fuller use of the graph topology and the low-frequency information from GNNs, and in turn outperforms GLNN that adopts node-to-node distillation on all six datasets. (2) While both low- and high-frequency distillation can work alone to improve the performance of vanilla MLPs, the former plays a \emph{primary} role and the latter a \emph{secondary} (auxiliary) role. More importantly, these two distillations are complementary to each other and can further improve performance on top of each other. (3) The FF-G2M (full model) considers both low- and high-frequency distillation and is capable of capturing full-frequency knowledge, and thus can far outperform GLNN on all six datasets.

\section{Conclusion}
In this paper, we factorize the knowledge learned by GNNs into low- and high-frequency components in the spectral and spatial domains and then conduct a comprehensive investigation on their roles played in GNN-to-MLP distillation. Our key finding is existing GNN-to-MLP distillation may suffer from a potential \emph{information drowning} problem, i.e., the high-frequency knowledge of the pre-trained GNNs may be overwhelmed by the low-frequency knowledge during distillation. Therefore, we propose a novel \textit{Full-Frequency GNN-to-MLP} (FF-G2M) knowledge distillation framework, which extracts both low- and high-frequency knowledge from GNNs and injects it into MLPs. As a simple but general framework, FF-G2M outperforms other leading methods across various GNN architectures and graph datasets. Limitations still exist; for example, this paper pays little attention to the special designs on teacher GNNs, and designing more expressive teachers to directly capture full-frequency knowledge may be another promising direction.

\section{Acknowledgement}
This work is supported in part by Ministry of Science and Technology of the People's Republic of China (No. 2021YFA1301603) and National Natural Science Foundation of China (No. U21A20427).

\bibliography{ref}

\clearpage
\renewcommand\thefigure{A\arabic{figure}}
\renewcommand\thetable{A\arabic{table}}
\setcounter{table}{0}
\setcounter{figure}{0}

\section*{Appendix}

\subsection*{A. Dataset Statistics} 
\emph{Six} publicly available graph datasets are used to evaluate the FF-G2M framework. An overview summary of the statistical characteristics of datasets is given in Tab.~\ref{tab:A1}. For the three small-scale datasets, namely Cora, Citeseer, and Pubmed, we follow the data splitting strategy in \cite{kipf2016semi}. For the three large-scale datasets, namely Coauthor-CS, Coauthor-Phy, and Amazon-Photo, we follow \cite{zhang2021graph,luo2021distilling} to randomly split the data, and each random seed corresponds to a different splitting.

\begin{table*}[!bp]
\begin{center}
\caption{Statistical information of the datasets.}
\label{tab:A1}
\resizebox{0.8\textwidth}{!}{
\begin{tabular}{lcccccc}

\toprule
\textbf{Dataset} & \textbf{Cora} & \textbf{Citeseer} & \textbf{Pubmed} & \textbf{Amazon-Photo} & \textbf{Coauthor-CS} & \textbf{Coauthor-Phy} \\ \midrule
$\#$ Nodes & 2708 & 3327 & 19717 & 7650 & 18333 & 34493 \\
$\#$ Edges & 5278 & 4614 & 44324 & 119081 & 81894 & 247962 \\
$\#$ Features & 1433 & 3703 & 500 & 745 & 6805 & 8415 \\
$\#$ Classes & 7 & 6 & 3 & 8 & 15 & 5 \\
Label Rate & 5.2\% & 3.6\% & 0.3\% & 2.1\% & 1.6\% & 0.3\% \\ \bottomrule

\end{tabular}} \vspace{-1em}
\end{center}
\end{table*}

\subsection*{B. Hyperparameters and Search Space}
All baselines and our approach are implemented based on the standard implementation in the DGL library \citep{wang2019dgl} using the PyTorch 1.6.0 library with Intel(R) Xeon(R) Gold 6240R @ 2.40GHz CPU and NVIDIA V100 GPU. The following hyperparameters are set for all datasets: learning rate $lr$ = 0.01 ($lr$ = 0.001 for Amazon-Photo); weight decay $decay$ = 5e-4; Maximum Epoch $E$ = 500; Layer number $L$ = 2 ($L$ = 3 for Cora and Amazon-Photo); distillation temperature $\tau_1$ = 1.0. The other dataset-specific hyperparameters are determined by a hyperparameter search tool - NNI for each dataset, including hidden dimension $F$, loss weight $\lambda$, dropout ratio $R$, and distillation temperature $\tau_2$. The hyperparameter search space is shown in Tab.~\ref{tab:A2}, and the model with the highest validation accuracy is selected for testing. The best hyperparameter choices for each dataset and GNN architecture are available in the supplementary materials.

\begin{table}[!htbp]
\begin{center}
\caption{Hyperparameter search space.}
\label{tab:A2}
\begin{tabular}{lc}
\toprule
Hyperparameters & Search Space \\ \midrule
Hidden Dimension $F$ & {[}128, 256, 512, 1024, 2048{]} \\ \hline
Loss Weight $\lambda$ & {[}0.0, 0.1, 0.2, 0.3, 0.4, 0.5{]} \\ \hline
Dropout Ratio $R$ & {[}0.3, 0.4, 0.5, 0.6{]} \\ \hline
Temperature $\tau_2$ & {[}1.0, 1.1, 1.2, 1.3, 1.4, 1.5{]} \\ \bottomrule
\end{tabular}
\end{center}
\end{table}

\subsection*{C. Pseudo-code}
The pseudo-code of the proposed FF-G2M framework is summarized in Algorithm.~\ref{algo:1}.

\subsection*{D. Complexity Analysis}
The training time complexity of FF-G2M mainly comes from two parts: (1) GNN training $\mathcal{O}(|\mathcal{V}|dF+|\mathcal{E}|F)$ and (2) knowledge distillation $\mathcal{O}(|\mathcal{E}|F)$, where $d$ and $F$ are the dimensions of input and hidden spaces. The total time complexity $\mathcal{O}(|\mathcal{V}|dF+|\mathcal{E}|F)$ is linear w.r.t the number of nodes $|\mathcal{V}|$ and edges $|\mathcal{E}|$, which is in the same order as GCN and GLNN. Besides, the inference time of FF-G2M can be reduced from $\mathcal{O}(|\mathcal{V}|dF+|\mathcal{E}|F)$ to $\mathcal{O}(|\mathcal{V}|dF)$, which is as fast as MLP, due to the removal of neighborhood aggregation.

\begin{algorithm}[!htbp]
	\caption{Algorithm for the \textit{Full-Frequency GNN-to-MLP} knowledge distillation framework}
	\label{algo:1}
	\begin{algorithmic}[1]
		\Require Feature Matrix: $\mathbf{X}$; Edge Set: $\mathcal{E}$; \# Epochs: $E$. 
		
		\Ensure Predicted Labels $\mathcal{Y}_U$ and network parameters of student MLPs $\{\mathbf{W}^{l}\}_{l=0}^{L-1}$.
		
		\State Randomly initialize the parameters of GNNs and MLPs.
		\State Compute GNN representations and pre-train the GNNs until convergence by $\mathcal{L}_{\mathrm{label}}$.
		\For{$epoch$ $\in$ \{0, 1, $\cdots$, $E-1$\}}
		\State Compute GNN representations $\{\mathbf{h}_i^{(L)}\}_{i=1}^N$ from the 
		\State pre-trained GNNs and freeze it;
		\State Compute MLP representations $\{\mathbf{z}_i^{(L)}\}_{i=1}^N$ and calcu-
		\State late the total loss $\mathcal{L}_{\mathrm{total}}$;
		\State Update model parameters $\{\mathbf{W}^{l}\}_{l=0}^{L-1}$ by back propa-
		\State gation of loss $\mathcal{L}_{\mathrm{total}}$.
		
		\EndFor
        \State Predicted labels $\mathcal{Y}_U$ for those unlabeled nodes $\mathcal{V}_U$.
		\State \textbf{return} Predicted labels $\mathcal{Y}_U$ and the network parameters of student MLPs $\{\mathbf{W}^{l}\}_{l=0}^{L-1}$.
	\end{algorithmic}
\end{algorithm}

\subsection*{E. Definitions on Evaluation Metrics}
As derived in Eq.~(\ref{equ:5}), the low-frequency knowledge can be represented by the sum of the target node feature and its neighborhood features in the spatial domain, which essentially encourages neighborhood smoothing. In this paper, the low-frequency knowledge is measured by the \textit{mean cosine similarity} of nodes with their 1-order neighbors, as follows
\begin{equation*}
    \textit{mean cosine similarit} = \frac{1}{|\mathcal{E}|} \sum_{i\in\mathcal{V}}\sum_{j\in\mathcal{N}_i} \frac{\mathbf{s}_i\cdot\mathbf{s}_j}{\left|\mathbf{s}_i\right|\left|\mathbf{s}_j\right|}
\end{equation*}
where $\mathbf{s}_i$ and $\mathbf{s}_j$ are the representations of node $v_i$ and $v_j$, and we set $\mathbf{s}_i=\mathbf{h}_i^{(L)}$ for teacher GNNs and $\mathbf{s}_i=\mathbf{z}_i^{(L)}$ for student MLPs. On the other hand, the high-frequency knowledge can be represented as the differences between the target node feature with its neighborhood features in the spatial domain. In this paper, the low-frequency knowledge is measured by the \textit{KL-divergence} between the pairwise distances of teacher GNNs and student MLPs, which is defined as
\begin{equation*}
\begin{aligned}
    \textit{KL-divergence} =\frac{1}{|\mathcal{E}|}\sum_{i\in\mathcal{V}}\sum_{j \in \mathcal{N}_i} \operatorname{softmax}\Big(\big|\mathbf{z}_j^{(L)}-\\\mathbf{z}_i^{(L)}\big|\Big) \log \frac{\operatorname{softmax}\Big(\big|\mathbf{z}_j^{(L)}-  \mathbf{z}_i^{(L)}\big|\Big)}{\operatorname{softmax}\Big(\big|\mathbf{h}_j^{(L)}-\mathbf{h}_i^{(L)}\big|\Big)}
\end{aligned}
\end{equation*}

\subsection*{F. More Qualitative Visualizations.} We consider GCNs as the teacher model and compare its visualization with that of vanilla MLPs, GLNN, and FF-G2M on the Citeseer dataset. We select a target node (id 557 for Citeseer) and analyze its relative position relationship with its neighbors in Fig.\ref{fig:A1}. \emph{The analyses presented in the paper for the Cora dataset still hold true for the Citeseer dataset.}

\begin{figure*}[!tbp]
	\begin{center}
		\subfigure[Pre-trained GCNs]{\includegraphics[width=0.24\linewidth]{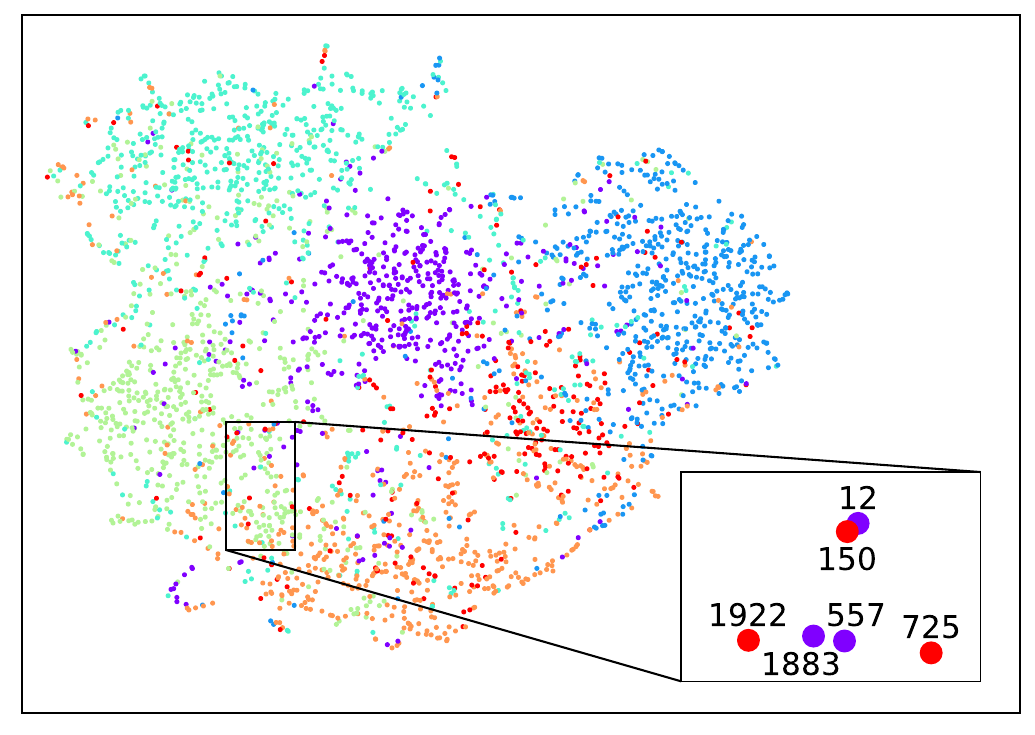}}
		\subfigure[Vanilla MLPs]{\includegraphics[width=0.24\linewidth]{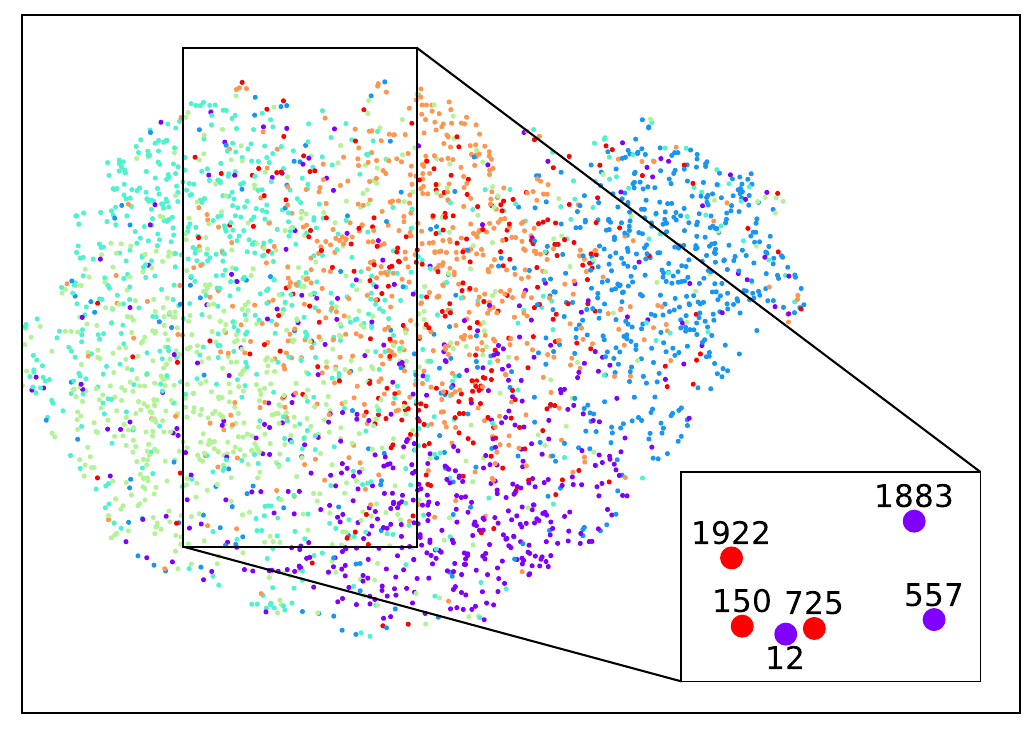}}
		\subfigure[Distilled MLPs (GLNN)]{\includegraphics[width=0.24\linewidth]{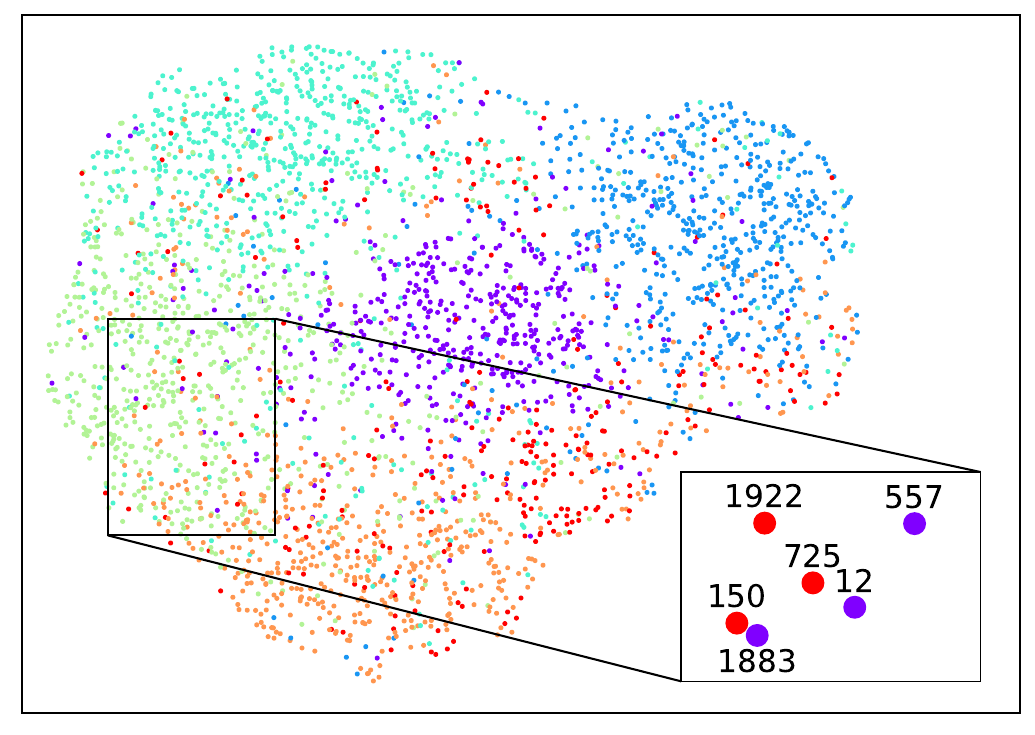}}
		\subfigure[Distilled MLPs (FF-G2M)]{\includegraphics[width=0.24\linewidth]{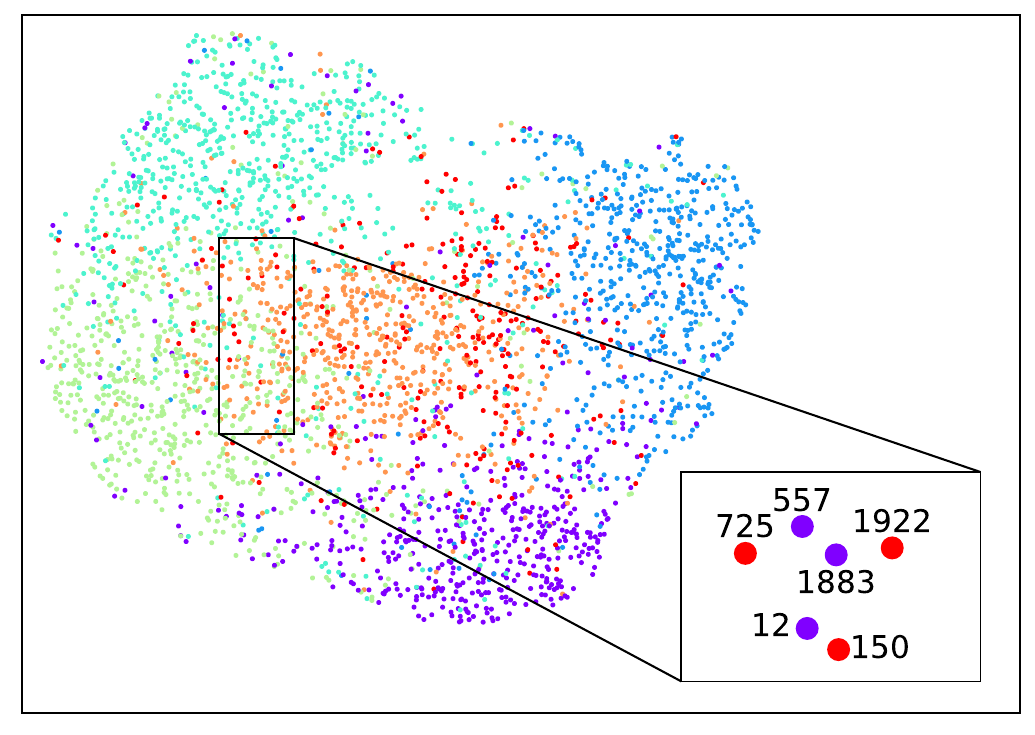}}
	\end{center}
	\vspace{-1em}
	\caption{Representation 2D-Visualizations (by UMAP \cite{mcinnes2018umap}) of the teacher model and three student models on Citeseer. Each node is colored by its ground-truth label, and the numbers around nodes denote the node ids.}
	\label{fig:A1}
\end{figure*}

\end{document}